\documentclass[10pt]{scrartcl}

\usepackage{makeidx}
\usepackage{graphicx}
\usepackage{amsthm}
\usepackage{amssymb}
\usepackage{amsmath}
\usepackage{amscd}

\newtheorem{theorem}{Theorem}

\newtheorem{proposition}{Proposition}
\newtheorem{lemma}{Lemma}

\newcommand{\commentout}[1]{}

\newcommand{\N}{\mathbb{N}}                    
\newcommand{\R}{\mathbb{R}}                    

\newcommand{\args}[1]{\mathop{\left( #1 \right)}} 
\newcommand{\inner}[1]{\mathop{\left\langle #1 \right\rangle}}
\newcommand{\norm}[1]{\mathop{\left\lVert #1 \right\rVert}}
\newcommand{\cbrace}[1]{\mathop{\left\{ #1 \right\}}}
\newcommand{\bracket}[1]{\mathop{\left[ #1 \right]}}

\newcommand{\absS}[2]{\mathop{\left\lvert #1 \right\rvert#2}} 
\newcommand{\argsS}[2]{\mathop{\left( #1 \right)#2}} 

\newcommand{\normS}[2]{\mathop{\left\lVert #1 \right\rVert#2}}

\renewcommand{\S}[1]{{\mathcal{#1}}}           
\def\vec#1{\mathchoice{\mbox{\boldmath$\displaystyle#1$}}
{\mbox{\boldmath$\textstyle#1$}}
{\mbox{\boldmath$\scriptstyle#1$}}
{\mbox{\boldmath$\scriptscriptstyle#1$}}}

\renewenvironment{cases}{%
\left\{\begin{array}{c@{\quad : \quad}l}}%
{%
\end{array}\right.}

\begin{document}

\title{Condorcet's Jury Theorem for Consensus Clustering and its Implications for Diversity}

\author{Brijnesh J.~Jain \\
 Technische Universit\"at Berlin, Germany\\
 e-mail: brijnesh.jain@gmail.com}
 
\date{}
\maketitle

\begin{abstract}
Condorcet's Jury Theorem has been invoked for ensemble classifiers to indicate that the combination of many classifiers can have better predictive performance than a single classifier. Such a theoretical underpinning is unknown for consensus clustering. This article extends Condorcet's Jury Theorem to the mean partition approach under the additional assumptions that a unique ground-truth partition exists and sample partitions are drawn from a sufficiently small ball containing the ground-truth. As an implication of practical relevance, we question the claim that the quality of consensus clustering depends on the diversity of the sample partitions. Instead, we conjecture that 
limiting the diversity of the mean partitions is necessary for controlling the quality. 
\end{abstract}

\section{Introduction}

Ensemble learning generates multiple models and combines them to a single consensus model to solve a learning problem. The assumption is that a consensus model performs better than an individual model or at least reduces the likelihood of selecting a model with inferior performance \cite{Polikar2009}. Examples of ensemble learning are classifier ensembles \cite{Dietterich2000,Kuncheva2004,Rokach2010,Zhou2012} and cluster ensembles (consensus clustering) \cite{Ghaemi2009,Strehl2002,VegaPons2011,Yang2014}. 

The assumptions on ensemble learning follow the idea of \emph{collective wisdom} that many heads are in general better than one. The idea of group intelligence applied to societies can be traced back to Aristotle and the philosophers of antiquity (see \cite{Waldron1995}) and has been recently revived by a number of publications, including James Surowiecki's book \emph{The Wisdom of Crowds} \cite{Surowiecki2005}. In his book, Surowiecki argues that not all crowds are wise, but in order to become wise, the crowd should comply to \emph{diversity} of opinion and other criteria. 

\subsubsection*{\textmd{Collective Wisdom of the Crowds.}}
One theoretical basis for \emph{collective wisdom} can be derived from Condorcet's Jury Theorem \cite{Condorcet1785}. The theorem refers to a jury of $n$ voters that need to reach a decision by majority vote. The assumptions of the simplest version of the theorem are:
(1) There are two alternatives; (2) one of both alternatives is correct; (3) voters decide independently; and (4) the probability $p$ of a correct decision is identical for every voter. If the voters are competent, that is $p > 0.5$, then Condorcet's Jury Theorem states that the probability of a correct decision by majority vote tends to one as the number $n$ of voters increases to infinity.

Condorcet's Jury Theorem has been generalized in several ways, because its assumptions are considered as rather restrictive and partly unrealistic (see e.g.~\cite{Berend1998} and references therein). Despite its practical limitations, the theorem has been used to indicate a theoretical justification of ensemble classifiers \cite{Kuncheva2004,Lam1997,Rokach2010}. In contrast to ensemble classifiers, such a theoretical underpinning is unknown for consensus clustering.

\subsubsection*{\textmd{Diversity of Opinion.}}
Though the importance of diversity for classifier ensembles has been recognized long before Surowiecki's book \cite{Krogh1995,Kuncheva2004}, some authors and scholars additionally employ Surowiecki's argument on diversity in retrospect for ensemble learning \cite{Rokach2010}. Inspired by the success of ensemble classifiers, diversity has also been promoted as a key factor for improving the quality of consensus clustering \cite{Hadjitodorov2006,Fern2008,Hu2016,Iam-On2011,Kuncheva2004b,Li2008,Pividori2016,Strehl2002,VegaPons2011,Yang2014}. However, a sound and consistent relationship between diversity and quality of consensus clustering could not yet be established. Instead, different empirical studies on the role of diversity in consensus clustering led to contradictory results \cite{Pividori2016}.

\subsubsection*{\textmd{Contribution.}}

This article addresses both problems: (1) the lack of a theoretical underpinning of consensus clustering with regard to the paradigm of collective wisdom; and (2) the lack of a consistent relationship between diversity and quality.  

With regard to the first problem, we extend Condorcet's Jury Theorem to the mean partition approach \cite{Dimitriadou2002,Domeniconi2009,Filkov2004,Franek2014,Gionis2007,Li2007,Strehl2002,Topchy2005,VegaPons2010}. Here, we consider the special case that the partition space is endowed with a metric induced by the Euclidean norm. Then the proposed theorem draws on the following assumptions: (1) there is a unique (possibly unknown) ground-truth partition $X_*$; and (2) sample partitions are drawn i.i.d.~from a sufficiently small ball containing $X_*$.

With regard to the second problem, this article analyzes the role of diversity in light of Condorcet's Jury Theorem within a special problem setting. We question the claim that the quality of consensus clustering depends on the diversity of the sample partitions generated by the cluster ensemble. Instead, we conjecture that limiting the diversity of the mean partitions is necessary to control the quality. This finding would explain the contradictory results of previous empirical studies on diversity as outlined and discussed in \cite{Pividori2016}. 

\subsubsection*{\textmd{Structure of the Paper.}}
The rest of this paper is structured as follows: Section 2 introduces background material. In Section 3 we present the extended version of Condorcet's Jury Theorem. Section 4 discusses the role of diversity and Section 5 concludes. Proofs are delegated to the appendix. 

\section{Background and Related Work}

\subsection{The Mean Partition Approach}

The goal is to group a set $\S{Z}= \cbrace{z_1, \ldots, z_m}$ of $m$ data points into $\ell$ clusters. The mean partition approach first clusters the same data set $\S{Z}$ several times using different settings and strategies of the same or different cluster algorithms. The resulting clusterings form a sample $\S{S}_n = \args{X_1, \ldots, X_n}$ of $n$ partitions $X_i \in \S{P}$ of data set $\S{Z}$. The mean partition approach aims at finding a consensus clustering that minimizes a sum-of-distances criterion from the sample partitions. In Section \ref{sec:P}, we specify the underlying partition space and in Section \ref{subsec:F} we present a formal definition of the mean partition approach.

\subsection{Context of the Mean Partition Approach}
We place the mean partition approach into the broader context of mathematical statistics. The motivation is that mathematical statistics offers a plethora of useful results, the consensus clustering literature seems to be unaware of. For example, the proof of Condorcet's Jury Theorem rests on results from statistical analysis of graphs \cite{Jain2016}. These results in turn are rooted on Fr\'echet's seminal monograph \cite{Frechet1948} and its follow-up research.  

\medskip

Since a meaningful addition of partitions is unknown, the mean partition approach emulates an averaging procedure by minimizing a sum-of-distances criterion. This idea is not new and has been studied in more general form for almost seven decades. In 1948, Fr\'echet first generalized the idea of averaging in metric spaces, where a well-defined addition is unknown. He showed that specification of a metric and a probability distribution is sufficient to define a mean element as measure of central tendency. The mean  of a sample of elements is any element that minimizes the sum of squared distances from all sample elements. Similarly, the expectation of a probability distribution minimizes an integral of the sum of squared distances from all elements of the entire space.

Since Fr\'echet's seminal work, mathematical statistics studied asymptotic and other properties of the mean element in abstract metric spaces. Examples include statistical analysis of shapes \cite{Bhattacharya2012,Dryden1998,Huckemann2010,Kendall1984}, complex objects \cite{Marron2014,Wang2007}, tree-structured data \cite{Feragen2013,Wang2007}, and graphs \cite{Ginestet2012,Jain2016}. 

The partition spaces defined in Section \ref{subsec:intrinsic-metric} can be regarded as a special case of graph spaces \cite{Jain2009,Jain2015}. Consequently, the geometric as well as statistical properties of graph spaces carry over to partition spaces. The proof of the proposed theorem rests on the orbit space framework \cite{Jain2009,Jain2015}, on the  mean partition theorem in graph spaces, and on asymptotic properties of the sample mean of graphs \cite{Jain2016} that have been adopted to partition spaces \cite{Jain2015c,Jain2016b}.

\section{Fr\'echet Functions on Partition Spaces}\label{sec:P}

This section first introduces partition spaces endowed with a metric induced by the Euclidean norm. Then we formalize the mean partition approach using Fr\'echet functions. 

Throughout this contribution, we assume that $\S{Z} = \cbrace{z_1, \ldots, z_m}$ is a set of $m$ data points to be clustered and $\S{C} = \cbrace{c_1, \ldots, c_\ell}$ is a set of $\ell$ cluster labels. 

\subsection{Partitions and their Representations}

Partitions usually occur in two forms, in a labeled and in an unlabeled form, where labeled partitions can be regarded as representations of unlabeled partitions. 

\medskip

We begin with describing labeled partitions. Let $\vec{1}_d \in \R^d$ denote the vector of all ones. Consider the set 
\[
\S{X} = \cbrace{\vec{X} \in [0,1]^{\ell \times m} \,:\, \vec{X}^T\vec{1}_\ell = \vec{1}_m},
\]
of matrices with elements from the unit interval and whose columns sum to one. A matrix $\vec{X} \in\S{X}$ represents a labeled (soft) partition of $\S{Z}$. The elements $x_{kj}$ of $\vec{X} = \args{x_{kj}}$ describe the degree of membership of data point $z_j$ to the cluster with label $c_k$. The columns $\vec{x}_{:j}$ of $\vec{X}$ summarize the membership values of the data points $z_j$ across all $\ell$ clusters. The rows $\vec{x}_{k:}$ of $\vec{X}$ represent the clusters $c_k$.

Next, we describe unlabeled partitions. Observe that the rows of a labeled partition $\vec{X}$ describe a cluster structure. Permuting the rows of $\vec{X}$ results in a labeled partition $\vec{X}'$ with the same cluster structure but with a possibly different labeling of the clusters. In cluster analysis, the particular labeling of the clusters is usually meaningless. What matters is the abstract cluster structure represented by a labeled partition. Since there is no natural labeling of the clusters, we define the corresponding unlabeled partition as the equivalence class of all labeled partitions that can be obtained from one another by relabeling the clusters. Formally, an unlabeled partition is a set of the form
\[
X = \cbrace{\vec{PX} \,:\, \vec{P} \in \Pi^\ell},
\] 
where $\Pi^\ell$ is the set of all ($\ell \times \ell$)-permutation matrices. 

In the following, we briefly call $X$ a \emph{partition} instead of unlabeled partition. In addition, any labeled partition $\vec{X}' \in X$ is called a \emph{representation} of partition $X$. By $\S{P}$ we denote the set of all (unlabeled) partitions with $\ell$ clusters over $m$ data points. Since some clusters may be empty, the set $\S{P}$ also contains partitions with less than $\ell$ clusters. Thus, we consider $\ell \leq m$ as the maximum number of clusters we encounter. 

A \emph{hard partition} $X \in \S{P}$ is a partition whose matrix representations take only binary membership values from $\cbrace{0,1}$. By $\S{P}^+ $ we denote the subset of all hard partitions. Note that the columns of representations of hard partitions are standard basis vectors from $\R^\ell$.

Though we are only interested in unlabeled partitions, we still need labeled partitions for two reasons: (1) computers can not easily and efficiently cope with unlabeled partitions unless the clusters carry labels in terms of number or names; and (2) using labeled partitions considerably simplifies derivation of theoretical results.

\subsection{Intrinsic Metric}\label{subsec:intrinsic-metric}

We endow the set $\S{P}$ of partitions with an intrinsic metric $\delta$ induced by the Euclidean norm such that $(\S{P}, \delta)$ becomes a geodesic space. The Euclidean norm for matrices $\vec{X} \in \S{X}$ is defined by
\[
\norm{\vec{X}}= \argsS{\sum_{k = 1}^\ell \sum_{j = 1}^m \absS{x_{kj}}{^2}}{^{1/2}}.
\]
The norm $\norm{\vec{X}}$ is also known as the Frobenius or Schur norm. We call $\norm{\vec{X}}$ Euclidean norm in order to emphasize the geometric properties of the partition space. The Euclidean norm induces the distance function 
\[
\delta(X, Y) = \min \cbrace{\norm{\vec{X} - \vec{Y}}\,:\, \vec{X} \in X, \vec{Y} \in Y}
\]
for all partitions $X, Y \in \S{P}$. Then the pair $\args{\S{P}, \delta}$ is a geodesic metric space \cite{Jain2015c}, Theorem 2.1. Suppose that $X$ and $Y$ are two partitions. Then 
\begin{align}\label{eq:delta<=norm}
\delta(X, Y) \leq \norm{\vec{X}-\vec{Y}}
\end{align}
for all representations $\vec{X} \in X$ and $\vec{Y} \in Y$. For some pairs of representations $\vec{X}' \in X$ and $\vec{Y}' \in Y$ equality holds in Eq.~\eqref{eq:delta<=norm}. In this case, we say that representations $\vec{X}'$ and $\vec{Y}'$ are in \emph{optimal position}. Note that pairs of representations in optimal position are not uniquely determined.

\subsection{Fr\'echet Functions}\label{subsec:F}

We first formalize the mean partition approach using Fr\'echet functions. Then we present the Mean Partition Theorem, which is of pivotal importance for gaining deeper insight into the theory of the mean partition approach \cite{Jain2016b}. Here, we apply the Mean Partition Theorem to define the concept of majority vote. In addition, the proof of the proposed theorem resorts to the properties stated in the Mean Partition Theorem.

\medskip

Let $(\S{P}, \delta)$ be a partition space endowed with the metric $\delta$ induced by the Euclidean norm. We assume that $Q$ is a probability dsitribution on $\S{P}$ with support $\S{S}_Q$.\footnote{The support of $Q$ is the smallest closed subset $\S{S}_Q \subseteq \S{P}$ such that $Q(\S{S}_Q) = 1$.} Suppose that $\S{S}_n = \args{X_1, X_2, \ldots, X_n}$ is a sample of $n$ partitions $X_i$ drawn i.i.d.~from the probability distribution $Q$. Then the Fr\'{e}chet function of $\S{S}_n$ is of the form
\begin{align*}
F_n: \S{P} \rightarrow \R, \quad Z \mapsto \frac{1}{n}\sum_{i=1}^n \delta\!\argsS{X_i, Z}{^2}.
\end{align*}
A mean partition of sample $\S{S}_n$ is any partition $M \in \S{P}$ satisfying
\[
F_n(M) = \min_{X \in \S{P}} F_n(X).
\]
Note that a mean partition needs not to be a member of the support. In addition, a mean partition exists but is not unique, in general \cite{Jain2015c}.

The Mean Partition Theorem proved in \cite{Jain2016b}  states that any representation $\vec{M}$ of a local minimum $M$ of $F_n$ is the standard mean of sample representations in optimal position with $\vec{M}$.

\begin{theorem}
Let $\S{S}_n = \args{X_1, \ldots, X_n} \in \S{P}^n$ be a sample of $n$ partitions. Suppose that $M \in \S{P}$ is a local minimum of the Fr\'echet function $F_n(Z)$ of $\S{S}_n$. Then every representation $\vec{M}$ of $M$ is of the form
\begin{align*}
\vec{M} = \frac{1}{n} \sum_{i=1}^n \vec{X}_{i},
\end{align*}
where the $\vec{X}_{i} \in X_i$ are in optimal position with $\vec{M}$.
\end{theorem}

Condorcet's original theorem is an asymptotical statement about the majority vote. To adopt this statement, we introduce the notion of expected partition. An expected partition of probability distribution $Q$ is any partition $M_Q \in \S{P}$ that minimizes the expected Fr\'{e}chet function 
\[
F_Q: \S{P} \rightarrow \R, \quad Z \mapsto \int_{\S{P}} \delta(X, Z)^2\, dQ(X).
\]
 As for the sample Fr\'echet function $F_n$, the minimum of the expected Fr\'{e}chet function $F_Q$ exists but but is not unique, in general \cite{Jain2015c}.

\section{Condorcet's Jury Theorem}

This section extends Condorcet's Jury Theorem to the partition space defined in Section \ref{subsec:intrinsic-metric}.

\subsection{The General Setting}

Theorem \ref{theorem:CJT} extends Condorcet's Jury Theorem for hard partitions. Generalization to arbitrary partitions is out of scope and left for future research.  

The general setting of Theorem \ref{theorem:CJT} is as follows: Let $\S{S}_n = \args{X_1, \ldots, X_n}$ be a sample of $n$ hard partitions $X_i \in \S{P}^+$ drawn i.i.d.~from a probability distribution $Q$. Each of the sample partitions $X_i$ has a vote on a given data point $z \in \S{Z}$ with probability $p_i(z)$ of being correct. The goal is to reach a final decision on data point $z$ by majority vote. Theorem \ref{theorem:CJT} makes an asymptotic statement about the correctness of the majority vote given the probabilities $p_i$ . 

To formulate Theorem \ref{theorem:CJT}, we need to define the concepts of vote and majority vote.  The majority vote is based on the mean partition of a sample and is not necessarily a hard partition. Since the mean partition itself votes, we introduce votes for arbitrary (soft and hard) partitions and later restrict ourselves to samples of hard partitions when defining the majority vote. 

\medskip

\noindent
\textbf{Assumption.} In the following, we assume existence of a possibly unknown but unique hard ground-truth partition $X_* \in \S{P}^+$. By $\vec{X}_{*}$ we denote an arbitrarily selected but fixed representation of $X_*$.

\subsection{Votes}

We model the vote of a partition $X \in \S{P}$ on a given data point $z \in \S{Z}$. The vote of $X$ on $z$ has two possible outcomes: The vote is correct if $X$ agrees on $z$ with the ground-truth $X_*$, and the vote is wrong otherwise. To model the vote of a partition, we need to specify what we mean by \emph{agreeing on a data-point with the ground-truth}. 

An agreement function of representation $\vec{X}$ of $X$ is a function of the form
\[
k_{\vec{X}}: \S{Z} \rightarrow [0,1], \quad z_j \mapsto \inner{\vec{x}_{:j}, \vec{x}^*_{:j}}
\]
where $\vec{x}_{:j}$ and $\vec{x}^*_{:j}$ are the $j$-th columns of the representations $\vec{X}$ and $\vec{X}_{*}$, respectively. A column of a matrix represents the membership values of the corresponding data point across all clusters. Then the value $k_{\vec{X}}\!\args{z_j}$ measures how strongly representation $\vec{X}$ agrees with the ground-truth $\vec{X}_{*}$ on data point $z_j$. If $X$ is a hard partition, then $k_{\vec{X}}(z) = 1$ if $z$ occurs in the same cluster of $\vec{X}$ and $\vec{X}_{*}$, and $k_{\vec{X}}(z) = 0$ otherwise.

The vote of representation $\vec{X}$ of partition $X$ on data point $z$ is defined by
\[
V_{\vec{X}}(z) = \mathbb{I}\cbrace{k_{\vec{X}}(z) > 0.5},
\]
where $\mathbb{I}\cbrace{b}$ is the indicator function that gives $1$ if the boolean expression $b$ is true, and $0$ otherwise. Observe that $k_{\vec{X}} = V_{\vec{X}}$ for hard partitions $X \in \S{P}^+$. 

Based on the vote of a representation we can define the vote of a partition. The vote of partition is a Bernoulli distributed random variable. We randomly select a representation $\vec{X}$ of partition $X$ in optimal position with $\vec{X}_{*}$. Then the vote $V_X(z)$ of $X$ on data point $z$ is $V_{\vec{X}}(z)$. By
\[
p_X(z) = \mathbb{P} \args{V_X(z) = 1}.
\]
we denote the probability of a correct vote of partition $X$ on data point $z$. Note that the probability $p_X(z)$ is independent of the particular choice of representation $\vec{X}_*$ of the ground-truth partition $X_*$. 

\subsection{Majority Vote}

We assume that $\S{S}_n = \args{X_1, \ldots, X_n}$ is a sample of $n$ hard partitions $X_i \in \S{P}^+$ drawn i.i.d.~from a cluster ensemble. We define a majority vote $V_n(z)$ of sample $\S{S}_n$ on $z$ as follows: We randomly select a mean partition $M$ of $\S{S}_n$ and then set the majority vote $V_n(z)$ on $z$ to the vote $V_M(z)$ of the chosen mean partition $M$.\footnote{Recall that a mean partition is not unique in general.}

It remains to show that the vote $V_M(z)$ of any mean partition $M$ of $\S{S}_n$ is indeed a majority vote. To see this, we invoke the Mean Partition Theorem. Any representation $\vec{M}$ of mean partition $M$ is of the form 
\[
\vec{M} = \frac{1}{n} \sum_{i=1}^n \vec{X}_{i} 
\]
where $\vec{X}_{i} \in X_i$ are representations in optimal position with $\vec{M}$. For a given data point $z_j \in \S{Z}$, the mean membership values are given by 
\[
\vec{m}_{:j} = \frac{1}{n} \sum_{i=1}^n\vec{x}_{:j}^{(i)}, 
\]
where $\vec{x}_{:j}^{(i)}$ denotes the $j$-th column of representation $\vec{X}_{i}$. Since the columns of $\vec{x}_{:j}^{(i)}$ are standard basis vectors, the elements $m_{kj}$ of the $j$-th column $\vec{m}_{:j}$ contain the relative frequencies with which data point $z_j$ occurs in cluster $c_k$. Then the vote $V_{\vec{M}}(z_j)$ is correct if and only if the agreement function of $\vec{M}$ satisfies
\[
k_{\vec{M}}(z_j) = \inner{\vec{m}_{:j}, \vec{x}^*_{:j}} > 0.5.
\]
This in turn implies that there is a majority $m_{kj} > 0.5$ for some cluster $c_k$, because $X_*$ is a hard partition by assumption.

\subsection{Condorcet's Jury Theorem}

Roughly, Condorcet's Jury Theorem states that the majority vote tends to be correct when the individual voters are independent and competent. In consensus clustering, the majority vote is based on mean partitions. Individual sample partitions $X_i$ are competent on data point $z \in \S{Z}$ if the probability of a correct vote on $z$ is given by $p_i(z) > 0.5$. In the spirit of Condorcet's Jury Theorem, we want to show that the probability $\mathbb{P}(h_n(z) = 1)$ of the majority vote $h_n(z)$ tends to one with increasing sample size $n$.

In general, mean partitions are neither unique nor converge to a unique expected partition. This in turn may result in a non-convergent sequence $(h_n(z))_{n\in \N}$ of majority votes for a given data points $z$. In this case, it is not possible to establish convergence in probability to the ground-truth. To cope with this problem, we demand that the sample partitions are all contained in a sufficiently small ball, called asymmetry ball. The \emph{asymmetry ball} $\S{A}_Z$ of partition $Z \in \S{P}$ is the subset of the form
\[
\S{A}_Z = \cbrace{X \in \S{P} \,:\, \delta\!\args{X, Z} \leq \alpha_Z/4},
\]
where $\alpha_Z$ is the \emph{degree of asymmetry} of $Z$ defined by
\[
\alpha_Z = \min \cbrace{\norm{\vec{Z} - \vec{P}\vec{Z}} \,:\, \vec{Z} \in Z \text{ and } \vec{P} \in \Pi \!\setminus\! \cbrace{\vec{I}}}.
\]
A partition $Z$ is asymmetric if $\alpha_Z > 0$. If $\alpha_Z = 0$ the partition $Z$ is called symmetric. Any partition whose representations have mutually distinct rows is an asymmetric partition. Conversely, a partition is symmetric if it has a representation with at least two identical rows. We refer to \cite{Jain2016a} for more details on asymmetric partitions.

By $\S{A}_Z^\circ$ we denote the largest open subset of $\S{A}_Z$. If $Z$ is symmetric, then $\S{A}_Z^\circ = \emptyset$ be definition. Thus, a non-empty set $\S{A}_Z^\circ$ entails that $Z$ is symmetric. 

A probability distribution $Q$ is \emph{homogeneous} if there is a partition $Z$ such that the support $\S{S}_Q$ of probability distribution $Q$ is contained in the asymmetry ball $\S{A}_Z^\circ$. A sample $\S{S}_n$ is said to be homogeneous if the sample partitions of $\S{S}_n$ are drawn from a homogeneous distribution $Q$.

Now we are in the position to present Condorcet's Jury Theorem for the mean partition approach. For a proof we refer to the appendix. 
\begin{theorem}[Condorcet's Jury Theorem]\label{theorem:CJT}
Let $Q$ be a probability measure on $\S{P}^+$ with support $\S{S}_Q$. Suppose the following assumptions hold:
\begin{enumerate}
\item 
There is a partition $Z\in \S{P}$ such that $X_* \in \S{A}_Z^\circ$ and $\S{S}_Q \subseteq \S{A}_Z^\circ$. 
\item 
Hard partitions $X_1, \ldots,X_n \in \S{P}^+$ are drawn i.i.d.~according to $Q$. 
\item Let $z \in \S{Z}$. Then $p_z = p_{X}(z)$ is constant for all $X\in \S{S_Q}$.
\end{enumerate}
Then 
\begin{align}\label{eq:condorcet:01}
\lim_{n \to \infty} \mathbb{P}\! \args{V_n(z) = 1} = 
\begin{cases}
1 & p_z > 0.5\\
0 & p_z < 0.5\\
0.5 & p_z = 0.5
\end{cases}
\end{align}
for all $z \in \S{Z}$. If $p_z > 0.5$ for all $z \in \S{Z}$, then we have
\begin{align}\label{eq:condorcet:02}
\lim_{n \to \infty} \mathbb{P}\Big(\delta\!\args{M_n, X_*} = 0\Big) = 1,
\end{align}
where $\argsS{M_n}{_{n \in \N}}$ is a sequence of mean partitions. 
\end{theorem}

\medskip

Equation \eqref{eq:condorcet:01} corresponds to Condorcet's original theorem for majority vote on a single data point and Eq.~\eqref{eq:condorcet:02} shows that the sequence of mean partitions converges almost surely to the ground-truth partition. Observe that almost sure convergence in Eq.~\eqref{eq:condorcet:02} also holds when the probabilities $p_z$ differ for different data points $z \in \S{Z}$. From the proof of Condorcet's Jury Theorem follows that the ground-truth partition $X_*$ is an expected partition almost surely and therefore takes the form as described in the Expected Partition Theorem \cite{Jain2016b}. 

\section{The Role of Diversity}

This section discusses the role of diversity under the assumptions of a special setting. We question the prevailing claim that quality of consensus clustering depends on the diversity of the sample partitions. Instead, we conjecture that limiting the diversity of the mean partition set is necessary to control the quality.

\subsection{Diversity and Quality}

This section specifies the notion of diversity and quality for the subsequent analysis. 

\subsubsection*{Diversity}
Suppose that $\S{S}$ is either a sample or a set consisting of $n$ elements $X_1, \ldots, X_n$. There are two common approaches to measure the diversity of $\S{S}$. Both approaches assume a dissimilarity function $\Delta: \S{P} \times \S{P} \rightarrow \R$. Typical choices for $\Delta$ are based on the adjusted Rand index, the Jaccard index, and the Normalized Mutual Information. The first approach averages the pairwise dissimilarities
\[
G(\S{S}) = \frac{1}{n^2} \sum_{i=1}^n \sum_{j=1}^n \Delta\!\args{X_i, X_j}
\]
and the second approach is the variation 
\[
F(M) = \frac{1}{n} \sum_{i=1}^n \Delta\!\args{X_i, M},
\]
where $M$ is either a medoid or mean partition of $\S{S}$. A simple calculation shows that both diversity measures are closely related by the inequality
\[
F(M) \leq \frac{1}{n} \sum_{i=1}^n F\!\args{X_i} = G(\S{S}).
\]
For the following analysis, we assume that $\Delta= \delta^2$. In this case, the diversity of a homogeneous sample (set) $\S{S}$ is bounded by 
\[
F(M) \leq G(\S{S}) \leq \frac{\alpha_Z}{4},
\]
where $\alpha_Z$ is the degree of asymmetry of a partition $Z$ whose asymmetry ball $\S{A}_Z$ includes $\S{S}$. We say, a sample (set) is diverse if it is not homogeneous. 

\subsubsection*{Quality}
Let $X_*$ be the ground-truth partition to be predicted as good as possible. For this, we need to specify the quality of a prediction, that is what we exactly mean by the term "as good as possible". We introduce the loss $L(X) = \delta\!\args{X, X_*}$ of a partition $X$ as a measure of how well $X$ predicts the ground-truth $X_*$.

To measure the quality of consensus clustering, we assume that $\S{S} \subseteq \S{P}$ is a sample or a closed set. Then 
\[
L^*(\S{S}) = \max \cbrace{L(X) \,:\, X \in \S{F}}
\]
is the worst-case loss of $\S{S}$. Consensus clustering generates a sample $\S{S}_n$ of sample partitions. The sample $\S{S}_n$ determines the non-empty and finite set $\S{F}_n$ of mean partitions. A mean algorithm picks a mean partition $M$ from $\S{F}_n$ to predict the ground-truth $X_*$. We define the quality of consensus clustering by the worst-case loss $L^*\!\args{\S{F}_n}$. We consider the worst-case loss rather than the average or expected loss to keep the subsequent analysis simple.

\subsection{The Role of Diversity}
Diversity of the sample partitions is a basic prerequisite for consensus clustering. Without diversity, consensus clustering reproduces the same partition. In this case, the resulting mean partition is merely another replication. 

Research on diversity in consensus clustering assumes that the quality of consensus clustering depends on the diversity of the sample. However, there is a significant gap in argument between the fact that diversity is a basic prerequisite to make sense of consensus clustering and the assumption that quality depends on diversity. The hope is that this gap can be bridged by an appropriate but currently unknown diversity measure. The proof of Theorem \ref{theorem:CJT} rejects this hope. 

Theorem \ref{theorem:CJT} includes the assumptions of Condorcet's original theorem and adds two further assumptions: (1) existence of a unique ground-truth partition, and (2) existence of a sufficiently small ball that contains the ground-truth partition as well as any sample partition. 

The first assumption transfers the notion of correct decision in Condorcet's original formulation to the context of consensus clustering. The second assumption demands homogeneity of the probability measure. At first glance, homogeneity seems to contradict the assumption that diversity of the sample is a decisive factor for improving the quality of consensus clustering. 

A closer look at the proof of Theorem \ref{theorem:CJT} shows that neither homogeneity nor diversity of the sample matters. What matters is that the mean partition is a consistent estimator of the ground-truth. Homogeneous probability distributions merely form one among several other classes of distributions for which consistency of the mean partition to a unique expected partition is ensured. Diversity of the sample can entail both, the desired consistency properties as well as diversity of the mean and expected partitions. In the latter case, Theorem \ref{theorem:CJT} fails and the mean partition returned by the consensus clustering method may have nothing in common with the ground-truth.

The claim is that diversity of the sample partitions is neither necessary nor sufficient for finding mean partitions close to the ground-truth, whereas limiting the diversity of the mean partition set is necessary but not sufficient to control the quality of consensus clustering.

\subsection{Loss Decomposition}

Inspired by the bias-variance decomposition in supervised learning, we discuss the different sources of error in consensus clustering to understand how we can improve its quality. For this, we call a cluster ensemble stable (unstable) if its mean partition set is homogeneous (diverse). Suppose that $\S{F}_n \subseteq \S{P}$ is a mean partition set. Then 
\[
L_*\!\args{\S{F}_n} = \min \cbrace{L(X) \,:\, X \in \S{F}_n}
\]
is the approximation error  of $\S{F}$. The approximation error tells us how close the best solution we can attain is to the ground-truth.  
Then the worst-case loss $L^*\!\args{\S{F}_n}$ of a mean partition set $\S{F}_n$ can be expressed as
\[
L^*\!\args{\S{F}_n} = \underbrace{\Big(L^*\!\args{\S{F}_n} \,-\, L_*\!\args{\S{F}_n}\Big)}_{\text{estimation error}} \quad +  \underbrace{\phantom{\Big(}L_*\!\args{\S{F}_n}\phantom{\Big)}}_{\text{approximation error}}\hspace{-1.5em}.
\]
The approximation error is caused by the choice and design of the underlying clustering ensemble method. The estimation error is caused by the random process of generating partitions. This error measures the difference between the loss in the worst and best case. If the set $\S{F}_n$ of mean partitions is homogeneous (diverse), then the estimation error is small (large).

For homogeneous distributions, the set $\S{F}_n$ is a singleton. Therefore, the estimation error is always zero, but the approximation error is likely to be large, when prior knowledge about the data is not considered. More generally, stable approaches guarantee a small estimation error and thereby impose a strong bias about the kind of cluster structure the underlying ensemble method is looking for. If the assumption on the cluster structure does not match the ground-truth, the approximation error will be large. The usual way to improve stable approaches is to incorporate competence by means of prior knowledge and domain expertise in order to hopefully reduce the approximation error. Under the special assumptions of Theorem \ref{theorem:CJT}, the estimation error is zero and the worst-case loss $L^*\!\args{M_n}$ converges almost surely to zero as the sample size $n$ tends to infinity. 

In contrast, unstable approaches result in large estimation errors, but may more likely have a lower approximation error than stable approaches. Such approaches are less biased about the kind of cluster structure the underlying ensemble method is looking for. Unstable  approaches shift the problem of assuming a certain kind of cluster structure in the data to the problem of identifying a mean partition close to the ground-truth. In other words, the bias on the kind of cluster structure is shifted to a bias on the choice of mean partition. The problem is that existing mean algorithms are unable to cope with the latter bias, because the Fr\'echet function contains no information about the ground-truth. Consequently, a mean algorithm regards the diverse sample means as equivalent solutions of the same minimization problem. Thus, obtaining a mean partition with low loss is merely due to chance. 

Another issue is that unstable approaches are prone to the Texas sharpshooter fallacy: Suppose a data set $\S{Z}$ is given for which the ground truth is known. Repeatedly apply a consensus clustering method to $\S{Z}$. In each trial, consensus clustering generates sample partitions and returns a mean partition. The mean partitions of the different trials can differ substantially. In this case, select the top $N$ mean partitions closest to the known ground-truth and claim that diversity improves quality. 

Finally, it is unclear how to improve the quality of unstable approaches without turning them into stable versions. A possible solution could be to derive the mean partitions of the second order, that is the mean partitions of the mean partition set. If the mean partition set is diverse, it can happen that mean partition sets of higher order remain diverse. Empirical results could shed light on this issue but require enumerating a subset of mean partitions, which is computationally expensive. Also incorporating prior knowledge will not improve the quality of unstable approaches, because the mean partitions are diverse. 

Though we analyzed the relationship between diversity and quality in a special setting, we assume that the main findings carry over to other consensus clustering approaches as well as other quality and diversity measures, unless they address the key issue of non-uniqueness of the mean and expected partition. We therefore conjecture that improved results obtained by unstable approaches are due to chance and not any systematic factor such as diversity of the sample partitions. To improve quality, we advocate to limit the diversity of the mean partitions and to incorporate prior knowledge.

\section{Conclusion}
This contribution extends Condorcet's Jury Theorem to partition spaces endowed with a metric induced by the Euclidean norm under the following additional assumptions: (i) existence of a hard ground-truth partition, and (ii) all sample partitions and the ground-truth are contained in some asymmetry ball. In light of the proposed theorem and its assumptions, we show that the quality of consensus clustering is independent of the diversity of the sample partitions. To improve quality, we advocate to limit diversity of the mean partitions to keep the estimation error low and incorporate prior knowledge to reduce the approximation error. We regard this finding as helpful for  devising better consensus clustering algorithms. Future research aims at extending the proposed theorem by relaxing assumptions (i) and (ii) and at empirically investigating the conjecture on diversity of the mean partitions.

\bigskip

\noindent
\textbf{Acknowledgment.}\small 
~B.~Jain was funded by the DFG Sachbeihilfe \texttt{JA 2109/4-1}.

\begin{appendix}

\small

\section{Proof of Theorem \ref{theorem:CJT}}

To prove Theorem \ref{theorem:CJT}, it is helpful to use a suitable representation of partitions. We suggest to represent partitions as points of some geometric space, called orbit space \cite{Jain2015c}. Orbit spaces are well explored, possess a rich geometrical structure and have a natural connection to Euclidean spaces \cite{Bredon1972,Jain2015,Ratcliffe2006}. 

\subsection{Partition Spaces}
We denote the natural projection that sends matrices to the partitions they represent by
\[
\pi: \S{X} \rightarrow \S{P}, \quad \vec{X} \mapsto \pi(\vec{X}) = X.
\]
The group $\Pi = \Pi^\ell$ of all ($\ell \times \ell$)-of all ($\ell \times \ell$)-permutation matrices is a discontinuous group that acts on $\S{X}$ by matrix multiplication, that is
\[
\cdot : \Pi \times \S{X} \rightarrow \S{X}, \quad (\vec{P}, \vec{X}) \mapsto \vec{PX}.
\]
The orbit of $\vec{X} \in \S{X}$ is the set $\bracket{\vec{X}} = \cbrace{\vec{PX} \,:\, \vec{P} \in \Pi}$. The orbit space of partitions is the quotient space $\S{X}/\Pi = \cbrace{\bracket{\vec{X}} \,:\, \vec{X} \in \S{X}}$ obtained by the action of the permutation group $\Pi$ on the set $\S{X}$. We write $\S{P} = \S{X}/\Pi$ to denote the partition space and $X \in \S{P}$ to denote an orbit $[\vec{X}] \in \S{X}/\Pi$. The natural projection $\pi: \S{X} \rightarrow \S{P}$ sends matrices $\vec{X}$ to the partitions $\pi(\vec{X}) = \bracket{\vec{X}}$ they represent. The partition space $\S{P}$ is endowed with the intrinsic metric $\delta$ defined by $\delta(X, Y) = \min \cbrace{\norm{\vec{X} - \vec{Y}} \,:\, \vec{X} \in X, \vec{Y} \in Y}$.

\subsection{Dirichlet Fundamental Domains}

We use the following notations: By $\overline{\S{U}}$ we denote the closure of a subset $\S{U} \subseteq \S{X}$, by $\partial \S{U}$ the boundary of $\S{U}$, and by $\S{U}^\circ$ the open subset $\overline{\S{U}} \setminus \partial \S{U}$. The action of permutation $\vec{P} \in \Pi$ on the subset $\S{U}\subseteq \S{X}$ is the set defined by $\vec{P}\,\S{U} = \cbrace{\vec{PX} \, :\, \vec{X} \in \S{U}}$. By $\Pi^* = \Pi \setminus \cbrace{\vec{I}}$ we denote the subset of ($\ell \times \ell$)-permutation matrices without identity matrix $\vec{I}$. 

\medskip

A subset $\S{F}$ of $\S{X}$ is a fundamental set for $\Pi$ if and only if $\S{F}$ contains exactly one representation $\vec{X}$ from each orbit $\bracket{\vec{X}} \in \S{X}/\Pi$. 
A fundamental domain of $\Pi$ in $\S{X}$ is a closed connected set $\S{F} \subseteq \S{X}$ that satisfies 
\begin{enumerate}
\item $\displaystyle\S{X} = \bigcup_{\vec{P} \in \Pi} \vec{P}\S{F}$
\item $\vec{P} \S{F}^\circ \cap \S{F}^\circ = \emptyset$ for all $\vec{P} \in \Pi^*$.
\end{enumerate}

\begin{proposition}
Let $\vec{Z}$ be a representation of an asymmetric partition $Z \in \S{P}$. Then 
\[
\S{D}_{\vec{Z}} = \cbrace{\vec{X} \in \S{X} \,:\, \norm{\vec{X} - \vec{Z}} \leq \norm{\vec{X} - \vec{PZ}} \text{ for all }\vec{P} \in \Pi}
\]
is a fundamental domain, called Dirichlet fundamental domain of $\vec{Z}$. 
\end{proposition}

\begin{proof}
 \cite{Ratcliffe2006}, Theorem 6.6.13. 
\end{proof}

\begin{lemma}\label{lemma:unique-interior}
Let $\S{D}_{\vec{Z}}$ be a Dirichlet fundamental domain of representation $\vec{Z}$ of an asymmetric partition $Z \in \S{P}$. Suppose that $\vec{X}$ and $\vec{X}'$ are two different representations of a partition $X$ such that $\vec{X}, \vec{X}' \in \S{D}_{\vec{Z}}$. Then $\vec{X}, \vec{X}' \in \partial \S{D}_{\vec{Z}}$.
\end{lemma}

\noindent
\proof \
\cite{Jain2015}, Prop.~3.13 and \cite{Jain2016a}, Prop.~A.2. \qed

\subsection{Multiple Alignments}

Let $\S{S}_n = \args{X_1, \ldots, X_n}$ be a sample of $n$ partitions $X_i \in \S{P}$. A multiple alignment of $\S{S}_n$ is an $n$-tuple $\mathfrak{X}= \args{\vec{X}_{1}, \ldots, \vec{X}_{n}}$
consisting of representations $\vec{X}_{i}\in X_{i}$. By
\[
\S{A}_n = \cbrace{\mathfrak{X} = \args{\vec{X}_{1}, \ldots, \vec{X}_{n}} \,:\, \vec{X}_{1} \in X_1, \ldots, \vec{X}_{n} \in X_n}
\]
we denote the set of all multiple alignments of $\S{S}_n$. A multiple alignment $\mathfrak{X}=\args{\vec{X}_{1}, \ldots, \vec{X}_{n}}$ is said to be in optimal position with representation $\vec{Z}$ of a partition $Z$, if all representations $\vec{X}_{i}$ of $\mathfrak{X}$ are in optimal position with $\vec{Z}$. The mean of a multiple alignment $\mathfrak{X} = \args{\vec{X}_{1}, \ldots, \vec{X}_{n}}$ is denoted by
\[
\vec{M}_{\mathfrak{X}} = \frac{1}{n} \sum_{i=1}^n \vec{X}_{i}.
\] 
An optimal multiple alignment is a multiple alignment that minimizes the function 
\[
f_n\!\args{\mathfrak{X}} = \frac{1}{n^2}\sum_{i=1}^n \sum_{j=1}^n \normS{\vec{X}_{i} - \vec{X}_{j}}{^2}.
\]
The problem of finding an optimal multiple alignment is that of finding a multiple alignment with smallest average pairwise squared distances in $\S{X}$. To show equivalence between mean partitions and an optimal multiple alignments, we introduce the sets of minimizers of the respective functions $F_n$ and $f_n$:
\begin{align*}
\S{M}(F_n) &= \cbrace{M\in \S{P} \,:\, F_n(M) \leq F_n(Z) \text{ for all } Z \in \S{P}}\\
\S{M}(f_n) &= \cbrace{\mathfrak{X} \in \S{A}_n \,:\, f_n(\mathfrak{X}) \leq f_n(\mathfrak{X}') \text{ for all } \mathfrak{X}' \in \S{A}_n}
\end{align*}
For a given sample $\S{S}_n$, the set $\S{M}(F_n)$ is the mean partition set and $\S{M}(f_n)$ is the set of all optimal multiple alignments. The next result shows that any solution of $F_n$ is also a solution of $f_n$ and vice versa.
\begin{theorem}\label{ theorem:equivalence:Fn-fn}
For any sample $\S{S}_n \in \S{P}^n$, the map
\[
\phi:\S{M}\!\args{f_n} \rightarrow \S{M}\!\args{F_n}, \quad \mathfrak{X} \mapsto \pi\!\args{\vec{M}_{\mathfrak{X}}}
\]
is surjective.
\end{theorem}

\proof \cite{Jain2016b}, Theorem 4.1. \qed

\subsection{Proof of Theorem \ref{theorem:CJT}}

Parts 1--8 show the assertion of Eq.~\eqref{eq:condorcet:01} and Part 9 shows the assertion of Eq.~\eqref{eq:condorcet:02}.

\paragraph*{\textbf{1}}
Without loss of generality, we pick a representation $\vec{X}_{*}$ of the ground-truth partition $X_*$. Let $\vec{Z}$ be a representation of $Z$ in optimal position with $\vec{X}_{*}$. By 
\[
\S{A}_{\vec{Z}} = \cbrace{\vec{X} \in \S{X} \,:\, \norm{\vec{X}-\vec{Z}} \leq \alpha_Z/4}
\]
we denote the asymmetry ball of representation $\vec{Z}$. By construction, we have $\vec{X}_{*} \in \S{A}_{\vec{Z}}$.

\paragraph*{\textbf{2}}
Since $\Pi$ acts discontinuously on $\S{X}$, there is a bijective isometry 
\[
\phi:\S{A}_{\vec{Z}} \rightarrow \S{A}_Z, \quad \vec{X} \mapsto \pi(\vec{X})
\]
according to \cite{Ratcliffe2006}, Theorem 13.1.1. 

\paragraph*{\textbf{3}}
From \cite{Jain2016a}, Theorem 3.1 follows that the mean partition $M$ of $\S{S}_n$ is unique. We show that $M \in \S{A}_Z$. Suppose that $\mathfrak{X} = \args{\vec{X}_{1}, \ldots, \vec{X}_{n}}$ is a multiple alignment in optimal position with $\vec{Z}$. Since $\phi:\S{A}_{\vec{Z}} \rightarrow \S{A}_Z$ is a bijective isometry, we have
\[
f_n\!\args{\mathfrak{X}} = \frac{1}{n^2}\sum_{i=1}^n \sum_{j=1}^n \normS{\vec{X}_{i} - \vec{X}_{j}}{^2} = \frac{1}{n^2}\sum_{i=1}^n \sum_{j=1}^n \delta\!\argsS{X_i, X_j}{^2} 
\]
showing that the multiple alignment $\mathfrak{X}$ is optimal. From Theorem \ref{ theorem:equivalence:Fn-fn} follows that 
\[
\vec{M} = \vec{M}_{\mathfrak{X}} = \frac{1}{n} \sum_{i=1}^n \vec{X}_{i}
\]
is a representation of a mean partition $M$ of $\S{S}_n$. Since $\S{A}_{\vec{Z}}$ is convex, we find that $\vec{M} \in \S{A}_{\vec{Z}}$ and therefore $M \in \S{A}_Z$.

\paragraph*{\textbf{4}}
From Part 1--3 of this proof follows that the multiple alignment $\mathfrak{X}$ is in optimal position with $\vec{X}_{*}$. We show that there is no other multiple alignment of $\S{S}_n$ with this property. Observe that $\S{A}_{\vec{Z}}$ is contained in the Dirichlet fundamental domain $\S{D}_{\vec{Z}}$ of representation $\vec{Z}$. Let $\S{S}_{\vec{Z}} = \phi(\S{S}_Q)$ be a representation of the support in $ \S{A}_{\vec{Z}}^\circ$. Then by assumption, we have $\S{S}_{\vec{Z}} \subseteq \S{A}_{\vec{Z}}^\circ \subset \S{D}_{\vec{Z}}$ showing that $\S{S}_{\vec{Z}}$ lies in the interior of $\S{D}_{\vec{Z}}$. From the definition of a fundamental domain together with Lemma \ref{lemma:unique-interior} follows that $\mathfrak{X}$ is the unique optimal alignment in optimal position with $\vec{X}_{*}$. 

\paragraph*{\textbf{5}}
With the same argumentation as in the previous part of this proof, we find that $\vec{M}$ is the unique representation of $M$ in optimal position with $\vec{X}_{*}$.

\paragraph*{\textbf{6}}
Let $z \in \S{Z}$ be a data point. Since $\vec{X}_{i} \in X_i$ is the unique representation in optimal position with $\vec{X}_{*}$, the vote of $X_i$ on data point $z$ is of the form $V_{X_i}(z) = V_{\vec{X}_{i}}(z)$ for all $i \in \cbrace{1, \ldots, n}$. With the same argument, we have $V_n(z) = V_{M}(z) = V_{\vec{M}}(z)$.

\paragraph*{\textbf{7}}
By $\vec{x}^{(i)}(z)$ we denote the column of $\vec{X}_{i}$ that represents $z$. By definition, we have 
\[
p_z = \mathbb{P}\args{V_{X_i}(z) = 1} = \mathbb{P}\args{\inner{\vec{x}^{(i)}(z), \vec{x}^*(z)} > 0.5}
\]
for all $i \in \cbrace{1, \ldots, n}$. Since $X_i$ and $X_*$ are both hard partitions, we find that
\[
\inner{\vec{x}^{(i)}(z), \vec{x}^*(z)} = \mathbb{I}\cbrace{\vec{x}^{(i)}(z) = \vec{x}^*(z)},
\] 
where $\mathbb{I}$ denotes the indicator function. 

\paragraph*{\textbf{8}}
From the Mean Partition Theorem follows that 
\[
\vec{m}(z) = \frac{1}{n} \sum_{i=1}^n \vec{x}^{(i)}(z)
\]
is the column of $\vec{M}$ that represents $z$. Then the agreement of $\vec{M}$ on $z$ is given by
\begin{align*}
k_{\vec{M}}(z) 
&= \inner{\vec{m}(z), \vec{x}^*(z)}\\ 
&= \frac{1}{n}\sum_{i=1}^n \inner{\vec{x}^{(i)}(z), \vec{x}^*(z)}\\
&= \frac{1}{n}\sum_{i=1}^n \mathbb{I}\cbrace{\vec{x}^{(i)}(z) = \vec{x}^*(z)}.
\end{align*}
Thus, the agreement $k_{\vec{M}}(z)$ counts the fraction of sample partitions $X_i$ that correctly classify $z$. Let
\[
p_n = \mathbb{P}\args{h_n(z) = 1} = \mathbb{P}\args{k_{\vec{M}}(z) > 0.5}
\]
denote the probability that the majority of the sample partitions $X_i$ correctly classifies $z$. Since the votes of the sample partitions are assumed to be independent, we can compute $p_n$ using the binomial distribution
\[
p_n = \sum_{i=r}^n \binom{n}{i} p^i (1-p)^{n-i},
\]
where $r = \lfloor n/2 \rfloor + 1$ and $\lfloor a \rfloor$ is the largest integer $b$ with $b \leq a$. Then the assertion of Eq.~\eqref{eq:condorcet:01} follows from \cite{Grofman1983}, Theorem 1. 

\paragraph*{\textbf{9}}
We show the assertion of Eq.~\eqref{eq:condorcet:02}. By assumption, the support $\S{S}_Q$ is contained in an open subset of the asymmetry ball $\S{A}_Z$. From \cite{Jain2016a}, Theorem 3.1 follows that the expected partition $M_Q$ of $Q$ is unique. Then the sequence $(M_n)_{n \in \N}$ converges almost surely to the expected partition $M_Q$ according to \cite{Jain2015c}, Theorem 3.1 and Theorem 3.3. From the first eight parts of the proof follows that the limit partition $M_Q$ agrees on any data point $z$ almost surely with the ground-truth partition $X_*$. This shows the assertion.

\end{appendix}

\end{document}